%% file: main.tex
\documentclass{article}

\usepackage[main, final]{neurips_2026}

\usepackage[utf8]{inputenc}
\usepackage[T1]{fontenc}
\usepackage{microtype}

\usepackage{amsmath,amssymb,amsthm}

\usepackage{graphicx}
\usepackage[caption=false]{subfig}
\usepackage{float}

\usepackage{booktabs}
\usepackage{multirow}
\usepackage{tabularx}

\usepackage{xcolor}
\usepackage{hyperref}
\usepackage[capitalise,noabbrev]{cleveref}

\usepackage{url}

\usepackage{nicefrac}

\usepackage{pifont}

\newtheorem{definition}{Definition}
\newtheorem{proposition}{Proposition}

\title{GF-Score: Certified Class-Conditional Robustness Evaluation with Fairness Guarantees}

\author{%
  Arya Shah\\
  IIT Gandhinagar\\
  \texttt{arya.shah@iitgn.ac.in}
  \And
  Kaveri Visavadiya\\
  IIT Gandhinagar\\
  \texttt{kaveri.visavadiya@iitgn.ac.in}
  \And
  Manisha Padala\\
  IIT Gandhinagar\\
  \texttt{manisha.padala@iitgn.ac.in}
}

\begin{document}

\maketitle

\input{sections/abstract}

\input{sections/introduction}

\input{sections/related_work}
\input{sections/method}

\input{sections/experimental_setup}

\input{sections/results}

\input{sections/discussion}

\bibliographystyle{plainnat}
\bibliography{references}

\appendix
\input{sections/appendix}

\end{document}

%% file: sections/abstract.tex
\begin{abstract}
Adversarial robustness is essential for deploying neural networks in safety-critical applications, yet standard evaluation methods either require expensive adversarial attacks or report only a single aggregate score that obscures how robustness is distributed across classes. We introduce the \emph{GF-Score} (GREAT-Fairness Score), a framework that decomposes the certified GREAT Score into per-class robustness profiles and quantifies their disparity through four metrics grounded in welfare economics: the Robustness Disparity Index (RDI), the Normalized Robustness Gini Coefficient (NRGC), Worst-Case Class Robustness (WCR), and a Fairness-Penalized GREAT Score (FP-GREAT). The framework further eliminates the original method's dependence on adversarial attacks through a self-calibration procedure that tunes the temperature parameter using only clean accuracy correlations. Evaluating 22 models from RobustBench across CIFAR-10 and ImageNet, we find that the decomposition is exact, that per-class scores reveal consistent vulnerability patterns (e.g., ``cat'' is the weakest class in 76\% of CIFAR-10 models), and that more robust models tend to exhibit greater class-level disparity. These results establish a practical, attack-free auditing pipeline for diagnosing where certified robustness guarantees fail to protect all classes equally. We release our code on \href{https://github.com/aryashah2k/gf-score}{GitHub}.
\end{abstract}

%% file: sections/introduction.tex
\section{Introduction}
\label{sec:introduction}

Deep neural networks are vulnerable to adversarial examples: imperceptible perturbations that cause confident misclassifications~\citep{szegedy2014intriguingpropertiesneuralnetworks, goodfellow2015explainingharnessingadversarialexamples}. This vulnerability poses serious risks in safety-critical settings such as autonomous driving and medical diagnosis, where a single misclassification can have catastrophic consequences. Adversarial training~\citep{madry2019deeplearningmodelsresistant} remains the dominant defense, and substantial progress has been tracked through standardized benchmarks like RobustBench~\citep{croce2021robustbenchstandardizedadversarialrobustness}, which ranks models by their accuracy under the AutoAttack ensemble~\citep{croce2020reliableevaluationadversarialrobustness}. Yet a fundamental tension exists between robustness and standard accuracy~\citep{tsipras2019robustnessoddsaccuracy}, and recent work on certified defenses~\citep{cohen2019certifiedadversarialrobustnessrandomized} has shifted attention toward provable guarantees rather than empirical attack evaluations alone.

A critical limitation of current evaluation practice is that robustness is almost always reported as a single aggregate number. Whether the metric is empirical robust accuracy or a certified lower bound, it averages over the entire test distribution and thereby conceals how robustness is distributed across classes. Several studies have shown that adversarial training induces pronounced class-wise performance gaps: certain classes become far more vulnerable than others under the same model~\citep{benz2021robustnessoddsfairnessempirical, xu2021robustfairfairnessadversarial, tian2021analysisapplicationsclasswiserobustness}. For instance, an autonomous perception system may appear globally robust while being nearly defenseless on pedestrian classes. Despite this, no existing framework provides \emph{certified, attack-free, per-class} robustness evaluation. The GREAT Score~\citep{li2024greatscoreglobalrobustness}, introduced at NeurIPS 2024, offers a certified global robustness bound using only generative model samples and forward passes, achieving roughly 2{,}000$\times$ speedup over attack-based methods. However, it reports only a single scalar, inheriting the same class-blindness problem. Concurrently, training-time fairness interventions~\citep{wei2023cfaclasswisecalibratedfair, sun2022improvingrobustfairnessbalance, li2023watimproveworstclassrobustness, zhang2024fairnessawareadversariallearning} address the disparity during model optimization but offer no tools for post-hoc auditing of already-deployed models.

In this paper, we introduce the \emph{GF-Score} (GREAT-Fairness Score), a framework that bridges this gap through three components. First, we decompose the GREAT Score into per-class certified robustness profiles by partitioning samples according to their ground-truth labels and computing class-conditional confidence margins. This decomposition is exact: the weighted sum of per-class scores recovers the aggregate score with zero numerical error. Second, we quantify the disparity of these per-class profiles through four metrics grounded in welfare economics and fairness theory: the Robustness Disparity Index (RDI), the Normalized Robustness Gini Coefficient (NRGC), Worst-Case Class Robustness (WCR), and a Fairness-Penalized GREAT Score (FP-GREAT). Third, we eliminate the original method's dependence on adversarial attacks for temperature calibration by introducing a self-calibration procedure that maximizes rank correlation with publicly available clean accuracies.

We evaluate the GF-Score on 22 robust models from RobustBench spanning CIFAR-10 (17 $\ell_2$ models) and ImageNet (5 $\ell_\infty$ models). Our experiments yield several notable findings. The class-conditional decomposition is \emph{exactly} consistent across all 22 models, confirming the mathematical validity of the approach. Per-class analysis reveals that the class ``cat'' is the most vulnerable in 76\% of CIFAR-10 models, while ``automobile'' is consistently the most robust, suggesting that class vulnerability is an intrinsic data property rather than a training artifact. We observe a positive correlation between aggregate robustness and the Robustness Disparity Index, providing new quantitative evidence for the tension between robustness and fairness identified by prior work~\citep{xu2021robustfairfairnessadversarial, benz2021robustnessoddsfairnessempirical}. Our attack-free self-calibration achieves a Spearman rank correlation of $\rho = 0.871$ on CIFAR-10 and $\rho = 1.000$ on ImageNet with RobustBench rankings, making the entire evaluation pipeline truly attack-free.

\begin{figure}[t]
    \centering
    \includegraphics[width=\linewidth]{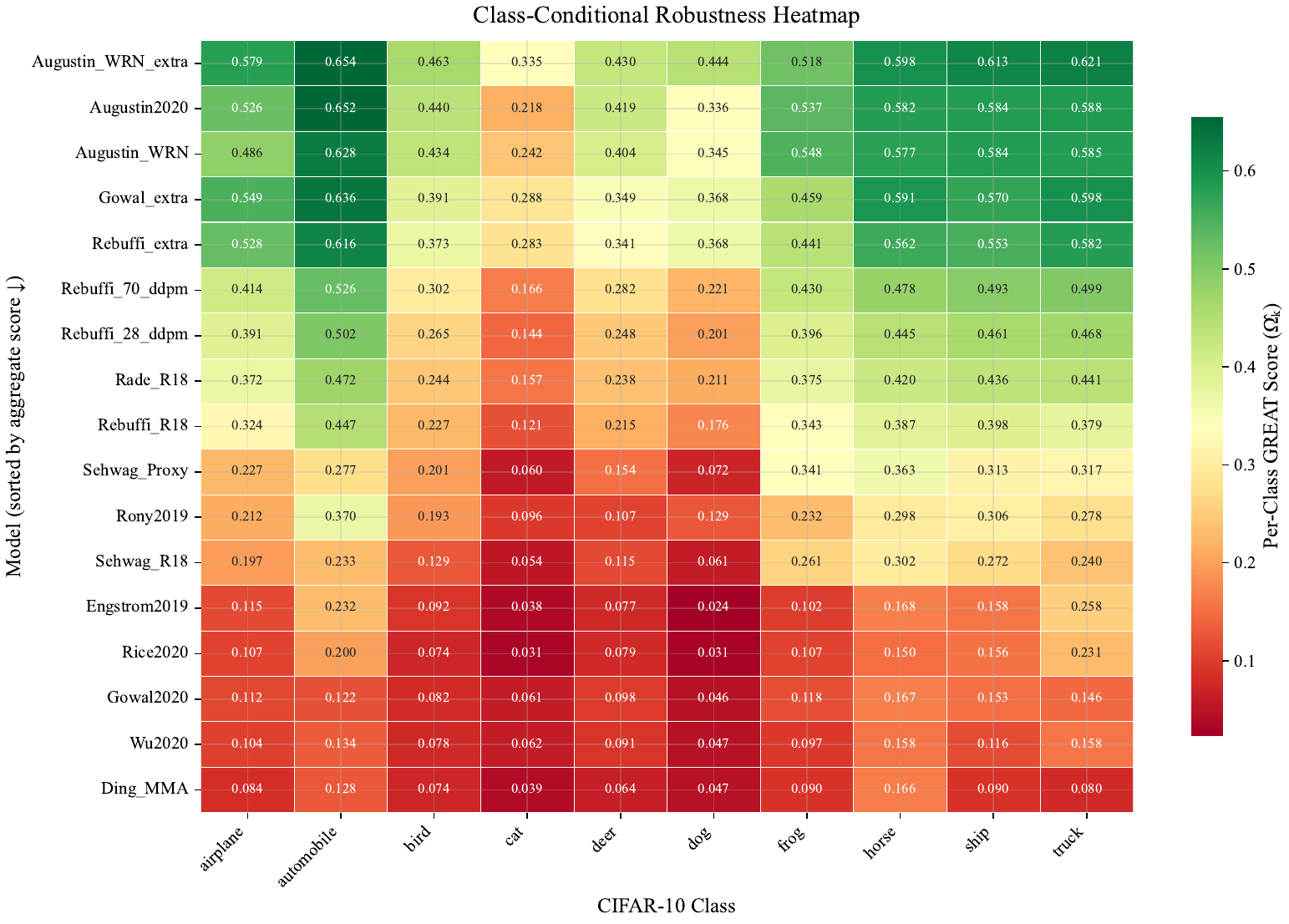}
    \caption{Per-class GREAT Scores for 17 CIFAR-10 models reveal substantial class-level robustness disparity hidden by aggregate scores. Each row is a model (sorted by aggregate GREAT Score); each column is a class. The class ``cat'' is consistently the most vulnerable (darkest column), while ``automobile'' is most often the most robust (10 of 17 models). Models with higher aggregate robustness (top rows) tend to exhibit \emph{greater} disparity between their strongest and weakest classes.}
    \label{fig:heatmap}
\end{figure}

In summary, we make the following contributions:
\begin{enumerate}
    \item We propose a \textbf{class-conditional decomposition} of the GREAT Score that preserves the certified lower-bound guarantee at per-class granularity, with formal concentration bounds (Propositions~\ref{prop:per_class_concentration} and~\ref{prop:rdi_concentration}).
    \item We introduce \textbf{four fairness-aware disparity metrics} (RDI, NRGC, WCR, FP-GREAT) grounded in welfare economics that quantify how robustness is distributed across classes.
    \item We propose an \textbf{attack-free self-calibration} procedure that replaces adversarial-attack-based temperature tuning with clean accuracy correlation, enabling fully attack-free evaluation.
    \item We conduct extensive experiments on \textbf{22 models across two benchmarks}, revealing consistent class vulnerability patterns and a quantifiable robustness-fairness tension that aggregate metrics conceal.
\end{enumerate}

%% file: sections/related_work.tex
\section{Related Work}
\label{sec:related_work}

Our work lies at the intersection of three active research areas: adversarial robustness evaluation, fairness in adversarial training, and inequality measurement in machine learning. We synthesize each area below and position our contribution relative to existing methods in \Cref{tab:comparison}.

\subsection{Adversarial Robustness Evaluation}
\label{sec:rw_evaluation}

Since the discovery that neural networks are vulnerable to imperceptible perturbations~\citep{szegedy2014intriguingpropertiesneuralnetworks, goodfellow2015explainingharnessingadversarialexamples}, a rich line of work has developed increasingly powerful attacks to evaluate robustness. Gradient-based methods such as PGD~\citep{madry2019deeplearningmodelsresistant} and the C\&W attack~\citep{carlini2017evaluatingrobustnessneuralnetworks} became standard tools, though \citet{athalye2018obfuscatedgradientsfalsesense} showed that many defenses merely obfuscated gradients rather than achieving true robustness. The AutoAttack ensemble~\citep{croce2020reliableevaluationadversarialrobustness} addressed this by combining complementary attack strategies into a reliable evaluation protocol, and RobustBench~\citep{croce2021robustbenchstandardizedadversarialrobustness} standardized model comparison through a public leaderboard.

In parallel, certified approaches emerged to provide provable robustness guarantees. Convex relaxation methods~\citep{wong2018provabledefensesadversarialexamples} and randomized smoothing~\citep{cohen2019certifiedadversarialrobustnessrandomized, lecuyer2019certifiedrobustnessadversarialexamples, salman2020provablyrobustdeeplearning} offer formal certificates that no perturbation within a given norm ball can change the prediction. \citet{li2023sokcertifiedrobustnessdeep} provide a comprehensive systematization of these approaches. However, all certified methods operate at the \emph{per-sample} level, producing local guarantees that must be aggregated to characterize a model's overall robustness. The GREAT Score~\citep{li2024greatscoreglobalrobustness} took a fundamentally different approach by defining a \emph{global} certified robustness metric over the data distribution using generative models, achieving strong rank correlation with RobustBench while requiring only forward passes. Our work extends the GREAT Score by decomposing this global metric into per-class components, revealing structure that the aggregate score conceals.

\subsection{Fairness in Adversarial Robustness}
\label{sec:rw_fairness}

The observation that adversarial training creates significant class-wise performance gaps was first documented empirically by \citet{benz2021robustnessoddsfairnessempirical}, who showed that robust accuracy can vary by over 30 percentage points across classes. \citet{xu2021robustfairfairnessadversarial} formalized this as the \emph{robust fairness} problem and proposed Fair Robust Learning (FRL) through class-level reweighting and remargin strategies. \citet{tian2021analysisapplicationsclasswiserobustness} further analyzed the phenomenon at KDD 2021, demonstrating that class-wise vulnerability patterns are systematic rather than random.

These findings inspired a wave of training-time interventions. \citet{wei2023cfaclasswisecalibratedfair} proposed class-wise calibrated adversarial configurations (CFA) that adapt attack strength per class. \citet{sun2022improvingrobustfairnessbalance} introduced balanced adversarial training (BAT) to equalize robustness across classes. \citet{li2023watimproveworstclassrobustness} optimized directly for worst-class robustness via WAT. More recently, \citet{zhang2024fairnessawareadversariallearning} framed the problem through distributionally robust optimization (FAAL), \citet{zhi2025fairclasswiserobustnessclass} proposed class-optimal distribution adversarial training (CODA), \citet{jin2025enhancingrobustfairnessconfusional} regularized the spectral norm of the robust confusion matrix at ICLR 2025, and \citet{lin2023hardadversarialexamplemining} proposed hard adversarial example mining to improve robust fairness. The connection to broader fairness literature is reinforced by \citet{sagawa2020distributionallyrobustneuralnetworks}, who showed that standard training can fail on minority groups and proposed GroupDRO for worst-group optimization. Several concurrent works continue to address class-wise disparity from complementary angles~\citep{amerehi2025narrowingclasswiserobustnessgaps, zhu2026reducingclasswiseperformancedisparity, Mou2025-he}.

A crucial observation is that \textbf{all of the above methods operate at training time}. They modify the adversarial training procedure to produce fairer models, but they do not provide tools for \emph{post-hoc auditing} of models that have already been trained and deployed. Our framework fills precisely this gap: it evaluates and quantifies class-level robustness disparity for any given model, without requiring retraining or adversarial attacks.

\subsection{Inequality Metrics and Welfare-Theoretic Fairness}
\label{sec:rw_inequality}

Our disparity metrics draw on a long tradition in welfare economics. The Gini coefficient has been the standard measure of distributional inequality for over a century, and the Rawlsian maximin principle~\citep{rawls1971theory} provides a philosophical foundation for prioritizing the worst-off group. \citet{Speicher_2018} were among the first to connect these classical inequality indices to algorithmic fairness, proposing a unified framework that encompasses individual and group fairness through generalized entropy indices and the Gini coefficient. \citet{cousins2021axiomatictheoryprovablyfairwelfarecentric} formalized welfare-centric machine learning axiomatically at NeurIPS 2021, establishing conditions under which welfare functions yield provably fair outcomes.

Despite this rich foundation, no prior work has applied formal inequality metrics to \emph{certified robustness bounds}. Existing fairness-aware robustness studies~\citep{xu2021robustfairfairnessadversarial, wei2023cfaclasswisecalibratedfair, zhang2024fairnessawareadversariallearning} use ad-hoc measures such as the gap between the best and worst class accuracy, without grounding them in established fairness theory. Our RDI, NRGC, WCR, and FP-GREAT metrics bridge this disconnect by applying principled inequality measures from welfare economics directly to certified per-class robustness scores, and we further provide formal concentration bounds on these metrics via Hoeffding's inequality~\citep{Hoeffding1963-sq}.

\subsection{Positioning of This Work}

\Cref{tab:comparison} summarizes how the GF-Score relates to prior work along five key dimensions. Unlike training-time fairness methods, our framework requires no model modification. Unlike existing evaluation methods, it provides per-class certified guarantees with formal disparity quantification. The combination of attack-free evaluation, class-conditional decomposition, and welfare-grounded fairness metrics is, to our knowledge, novel.

\begin{table}[t]
    \caption{Comparison of the GF-Score with prior methods across five key dimensions. \textbf{Certified}: provides provable robustness guarantees. \textbf{Per-class}: offers class-level granularity. \textbf{Attack-free}: does not require adversarial attacks. \textbf{Post-hoc}: applicable to already-trained models. \textbf{Fairness metrics}: includes formal disparity quantification grounded in established theory.}
    \label{tab:comparison}
    \centering
    \small
    \begin{tabular}{@{}lccccc@{}}
        \toprule
        \textbf{Method} & \textbf{Certified} & \textbf{Per-class} & \textbf{Attack-free} & \textbf{Post-hoc} & \textbf{Fairness metrics} \\
        \midrule
        AutoAttack~\citep{croce2020reliableevaluationadversarialrobustness} & \ding{55} & \ding{55} & \ding{55} & \ding{51} & \ding{55} \\
        RobustBench~\citep{croce2021robustbenchstandardizedadversarialrobustness} & \ding{55} & \ding{55} & \ding{55} & \ding{51} & \ding{55} \\
        Rand.\ Smoothing~\citep{cohen2019certifiedadversarialrobustnessrandomized} & \ding{51} & \ding{51}\textsuperscript{*} & \ding{51} & \ding{51} & \ding{55} \\
        GREAT Score~\citep{li2024greatscoreglobalrobustness} & \ding{51} & \ding{55} & \ding{51} & \ding{51} & \ding{55} \\
        FRL~\citep{xu2021robustfairfairnessadversarial} & \ding{55} & \ding{51} & \ding{55} & \ding{55} & \ding{55} \\
        CFA~\citep{wei2023cfaclasswisecalibratedfair} & \ding{55} & \ding{51} & \ding{55} & \ding{55} & \ding{55} \\
        WAT~\citep{li2023watimproveworstclassrobustness} & \ding{55} & \ding{51} & \ding{55} & \ding{55} & \ding{55} \\
        BAT~\citep{sun2022improvingrobustfairnessbalance} & \ding{55} & \ding{51} & \ding{55} & \ding{55} & \ding{55} \\
        FAAL~\citep{zhang2024fairnessawareadversariallearning} & \ding{55} & \ding{51} & \ding{55} & \ding{55} & \ding{55} \\
        CODA~\citep{zhi2025fairclasswiserobustnessclass} & \ding{55} & \ding{51} & \ding{55} & \ding{55} & \ding{55} \\
        CSR~\citep{jin2025enhancingrobustfairnessconfusional} & \ding{55} & \ding{51} & \ding{55} & \ding{55} & \ding{55} \\
        \midrule
        \textbf{GF-Score (Ours)} & \ding{51} & \ding{51} & \ding{51} & \ding{51} & \ding{51} \\
        \bottomrule
        \multicolumn{6}{@{}l}{\textsuperscript{*}\footnotesize Per-sample certificates can be aggregated per class, but no prior work has done so with disparity metrics.}
    \end{tabular}
\end{table}

%% file: sections/method.tex
\section{Methodology}
\label{sec:method}

We present the GF-Score framework in three parts. \Cref{sec:decomposition} introduces the class-conditional decomposition of the GREAT Score with its consistency guarantee. \Cref{sec:metrics} defines four disparity metrics grounded in welfare economics. \Cref{sec:calibration} describes the attack-free self-calibration procedure that makes the entire pipeline independent of adversarial attacks. Formal concentration bounds and their proofs are deferred to \Cref{sec:theoretical}.

\subsection{Preliminaries: The GREAT Score}
\label{sec:prelim}

We briefly recall the GREAT Score~\citep{li2024greatscoreglobalrobustness}. Let $f: \mathcal{X} \to \mathbb{R}^K$ be a classifier mapping inputs to $K$-dimensional logits. Given a generative model $G$ that approximates the data distribution, the GREAT Score defines a certified global robustness metric as
\begin{equation}
    \Omega(f) = \mathbb{E}_{z \sim p_z}\bigl[g\bigl(G(z)\bigr)\bigr],
    \label{eq:great_score}
\end{equation}
where $g(x) = \sqrt{\pi/2} \cdot \max\bigl\{\sigma(f(x))_{y} - \max_{j \neq y} \sigma(f(x))_j,\; 0\bigr\}$ is the local robustness score, $\sigma$ denotes the sigmoid or softmax activation scaled by a temperature parameter $T$, and $y$ is the ground-truth label of $x$. The local score $g(x)$ provides a certified lower bound on the $\ell_2$ perturbation required to change the prediction at $x$~\citep{li2024greatscoreglobalrobustness}. Given $N$ i.i.d.\ samples, the finite-sample estimator is $\hat{\Omega}(f) = \frac{1}{N}\sum_{i=1}^{N} g(x_i)$.

\subsection{Class-Conditional Decomposition}
\label{sec:decomposition}

The aggregate GREAT Score averages the local robustness scores over all samples irrespective of their class membership. We propose to partition the evaluation set by ground-truth labels to obtain per-class robustness profiles.

\begin{definition}[Per-Class GREAT Score]
\label{def:per_class}
Let $\mathcal{S}_k = \{x_i : y_i = k\}$ denote the set of $n_k$ samples belonging to class $k \in \{1, \ldots, K\}$. The per-class GREAT Score for class $k$ is
\begin{equation}
    \hat{\Omega}_k(f) = \frac{1}{n_k} \sum_{i : y_i = k} g(x_i).
    \label{eq:per_class_score}
\end{equation}
\end{definition}

Since $g(x_i)$ inherits the certified lower-bound property of the original GREAT Score for each sample, $\hat{\Omega}_k(f)$ is the average certified robustness bound restricted to class $k$. The following result shows that this decomposition is exact.

\textbf{Decomposition consistency.} The aggregate score can be recovered from the per-class scores as a weighted average:
\begin{equation}
    \hat{\Omega}(f) = \sum_{k=1}^{K} \frac{n_k}{N} \hat{\Omega}_k(f),
    \label{eq:decomposition}
\end{equation}
where $N = \sum_{k=1}^{K} n_k$ is the total number of samples. This identity holds exactly by linearity of the mean: each sample contributes to exactly one class partition, and the weighted recombination recovers the global average. We verify empirically that this identity holds with zero numerical error across all 22 models (\Cref{sec:exp_decomposition}). Concentration bounds for the per-class estimates are established in \Cref{sec:theoretical}.

\subsection{Robustness Disparity Metrics}
\label{sec:metrics}

Given the $K$-dimensional vector of per-class scores $(\hat{\Omega}_1, \ldots, \hat{\Omega}_K)$, we define four complementary metrics that each capture a different facet of the robustness distribution. These metrics are grounded in welfare economics and adapted to the robustness evaluation setting.

\begin{definition}[Robustness Disparity Index (RDI)]
\label{def:rdi}
\begin{equation}
    \mathrm{RDI}(f) = \max_{k} \hat{\Omega}_k(f) - \min_{k} \hat{\Omega}_k(f).
    \label{eq:rdi}
\end{equation}
The RDI measures the range of per-class robustness scores. It is zero if and only if all classes have equal robustness, and is bounded above by $\sqrt{\pi/2} \approx 1.253$. This metric adapts the Max Group Disparity principle commonly used in fairness auditing.
\end{definition}

\begin{definition}[Normalized Robustness Gini Coefficient (NRGC)]
\label{def:nrgc}
\begin{equation}
    \mathrm{NRGC}(f) = \frac{\displaystyle\sum_{i=1}^{K}\sum_{j=1}^{K} |\hat{\Omega}_i - \hat{\Omega}_j|}{2K^2 \cdot \bar{\Omega}},
    \label{eq:nrgc}
\end{equation}
where $\bar{\Omega} = \frac{1}{K}\sum_{k=1}^{K} \hat{\Omega}_k$ is the mean per-class score. The NRGC adapts the Gini coefficient to robustness evaluation: it lies in $[0, 1)$, with 0 indicating perfect equality and values approaching 1 indicating maximal concentration. Unlike RDI, the NRGC captures the full shape of the distribution rather than only the extremes, making it more informative when $K > 2$.
\end{definition}

\begin{definition}[Worst-Case Class Robustness (WCR)]
\label{def:wcr}
\begin{equation}
    \mathrm{WCR}(f) = \min_{k} \hat{\Omega}_k(f).
    \label{eq:wcr}
\end{equation}
Grounded in the Rawlsian maximin principle~\citep{rawls1971theory}, WCR gives the certified robustness level guaranteed for \emph{every} class. A model passes a fairness audit only if $\mathrm{WCR}(f) \geq \tau$ for some application-specific threshold $\tau$. This metric is particularly relevant for safety-critical deployments where no class can be left unprotected.
\end{definition}

\begin{definition}[Fairness-Penalized GREAT Score (FP-GREAT)]
\label{def:fp_great}
\begin{equation}
    \mathrm{FP\text{-}GREAT}(f; \lambda) = \bar{\Omega}(f) - \lambda \cdot \mathrm{RDI}(f),
    \label{eq:fp_great}
\end{equation}
where $\lambda \geq 0$ controls the fairness penalty weight. When $\lambda = 0$, the metric reduces to the mean per-class GREAT Score (no fairness consideration). As $\lambda$ increases, models with high disparity are penalized more heavily. This formulation adapts the Inequality-Adjusted Human Development Index (IHDI) used by the United Nations Development Programme, which discounts aggregate welfare by an inequality measure.
\end{definition}

\textbf{Complementarity of metrics.} The four metrics serve distinct roles. RDI flags the existence of disparity; NRGC quantifies its severity across the full distribution; WCR identifies the weakest link; FP-GREAT produces an adjusted ranking. Together, they provide a comprehensive robustness-fairness profile for any model. A concentration bound for the empirical RDI is provided in \Cref{sec:theoretical}.

\subsection{Attack-Free Self-Calibration}
\label{sec:calibration}

The original GREAT Score framework uses the C\&W attack~\citep{carlini2017evaluatingrobustnessneuralnetworks} to calibrate the temperature parameter $T$ in the softmax/sigmoid activation by maximizing rank correlation between GREAT Scores and attack-based robustness rankings. This dependence on adversarial attacks undermines the computational advantages of the certified approach. We propose a fully attack-free alternative.

\textbf{Key insight.} There exists a well-established monotonic relationship between adversarial robustness and clean accuracy~\citep{tsipras2019robustnessoddsaccuracy}: models that are more robust tend to achieve higher clean accuracy within the same defense family. We exploit this relationship to calibrate $T$ using only publicly available clean accuracy values from RobustBench~\citep{croce2021robustbenchstandardizedadversarialrobustness}, which require no adversarial computation.

\begin{definition}[Accuracy-Correlation Self-Calibration]
\label{def:self_cal}
Given $M$ models $\{f_m\}_{m=1}^{M}$ with known clean accuracies $\{a_m\}_{m=1}^{M}$, the self-calibrated temperature is
\begin{equation}
    T^* = \arg\max_{T \in \mathcal{T}} \; \rho_s\!\left(\bigl\{\hat{\Omega}^{(T)}(f_m)\bigr\}_{m=1}^{M},\; \bigl\{a_m\bigr\}_{m=1}^{M}\right),
    \label{eq:self_calibration}
\end{equation}
where $\rho_s$ denotes Spearman's rank correlation coefficient~\citep{spearman1904proof} and $\mathcal{T}$ is the search space.
\end{definition}

\textbf{Optimization procedure.} We employ a two-phase grid search. The coarse phase evaluates temperatures in $[0.01, 10.0]$ with step size $0.1$. The fine phase refines the search in a neighborhood of the best coarse temperature with step size $0.001$. This procedure is computationally inexpensive since it only requires recomputing softmax/sigmoid outputs from cached logits (no additional forward passes).

\textbf{Ranking stability calibration.} As a complementary approach, we also consider a stability-based calibration that selects the temperature at which per-class score rankings are most stable under small perturbations:
\begin{equation}
    T^*_{\mathrm{stab}} = \arg\max_{T \in \mathcal{T}} \min_{T' \in [T - \delta_T, T + \delta_T]} \rho_s\!\left(\mathrm{rank}(\hat{\boldsymbol{\Omega}}^{(T)}),\; \mathrm{rank}(\hat{\boldsymbol{\Omega}}^{(T')})\right),
    \label{eq:stability_calibration}
\end{equation}
where $\delta_T$ is a small perturbation window and $\hat{\boldsymbol{\Omega}}^{(T)} = (\hat{\Omega}_1^{(T)}, \ldots, \hat{\Omega}_K^{(T)})$ is the vector of per-class scores at temperature $T$. This criterion ensures that the chosen temperature produces robust rankings that do not fluctuate with minor parameter changes.

\textbf{Complete pipeline.} Combining the three components, the GF-Score evaluation pipeline proceeds as follows: (1) collect logits via forward passes on test or generated samples; (2) partition samples by class and compute per-class GREAT Scores (\Cref{def:per_class}); (3) compute disparity metrics (Definitions~\ref{def:rdi}--\ref{def:fp_great}); (4) self-calibrate the temperature (\Cref{def:self_cal}) and recompute all scores at $T^*$. The entire pipeline requires only forward passes through the classifier, with no adversarial attack computation at any stage.

%
%
%
%
%
%
%
%

%% file: sections/experimental_setup.tex
\section{Experimental Setup and Theoretical Analysis}
\label{sec:exp_setup_theory}

\subsection{Setup}
\label{sec:exp_setup}

\textbf{Models and datasets.} We evaluate 22 adversarially robust models from RobustBench~\citep{croce2021robustbenchstandardizedadversarialrobustness}: 17 CIFAR-10~\citep{krizhevsky2009learning} models under the $\ell_2$ threat model and 5 ImageNet~\citep{russakovsky2015imagenet} models under $\ell_\infty$. CIFAR-10 models span diverse training methods including PGD-AT~\citep{madry2019deeplearningmodelsresistant}, TRADES~\citep{zhang2019trades}, AWP~\citep{wu2020adversarial}, data augmentation with DDPM~\citep{rebuffi2021fixing}, and MMA training~\citep{ding2020mma}. ImageNet models include adversarially trained ResNet and WideResNet variants~\citep{salman2020provablyrobustdeeplearning, wong2020fast}. We use the standard test sets (10K for CIFAR-10; 50K for ImageNet) with GAN-generated samples for GREAT Score computation following the protocol of \citet{li2024greatscoreglobalrobustness}.

\textbf{Implementation.} All evaluations use only forward passes with cached logits. Temperature is set to $T=1.0$ (uncalibrated) or $T^*$ (self-calibrated). We use sigmoid activation for CIFAR-10 and softmax for ImageNet, matching the original GREAT Score setup. The fairness penalty is set to $\lambda = 0.5$ for FP-GREAT. Full model lists and per-class breakdowns are in Appendix~\ref{app:full_results}.

\subsection{Theoretical Analysis}
\label{sec:theoretical}

We establish formal concentration guarantees for the per-class GREAT Scores and the derived disparity metrics.

\begin{proposition}[Per-Class Concentration Bound]
\label{prop:per_class_concentration}
Since $g(x) \in [0, \sqrt{\pi/2}]$ is bounded, Hoeffding's inequality~\citep{Hoeffding1963-sq} gives, for each class $k$ and any $\epsilon > 0$:
\begin{equation}
    \Pr\bigl[|\hat{\Omega}_k - \Omega_k| \geq \epsilon\bigr] \leq 2\exp\!\left(-\frac{2 n_k \epsilon^2}{\pi/2}\right).
    \label{eq:hoeffding_class}
\end{equation}
Setting $\delta_k = 2\exp(-2n_k\epsilon^2 / (\pi/2))$ and applying a union bound over $K$ classes, with probability at least $1 - \delta$, simultaneously for all $k$:
\begin{equation}
    |\hat{\Omega}_k - \Omega_k| \leq \sqrt{\frac{\pi \log(2K/\delta)}{4 n_k}}.
    \label{eq:concentration_all}
\end{equation}
\end{proposition}

\begin{proof}
The local score $g(x)$ is bounded in $[0, \sqrt{\pi/2}]$ by construction (the confidence margin lies in $[0, 1]$ and is scaled by $\sqrt{\pi/2}$). For class $k$, the $n_k$ samples are i.i.d.\ draws from the class-conditional distribution, so Hoeffding's inequality applies directly with range parameter $b - a = \sqrt{\pi/2}$. Setting $\delta_k = \delta/K$ for each class and inverting for $\epsilon$ yields the stated bound. The union bound ensures all $K$ inequalities hold simultaneously.
\end{proof}

\textbf{Interpretation.} With $n_k = 1{,}000$ samples per class, $K = 10$ classes, and $\delta = 0.05$, the bound gives $|\hat{\Omega}_k - \Omega_k| \leq 0.069$ simultaneously for all classes.

\begin{proposition}[RDI Concentration Bound]
\label{prop:rdi_concentration}
Let $n_{\min} = \min_k n_k$. Under the conditions of \Cref{prop:per_class_concentration}, with probability at least $1 - \delta$:
\begin{equation}
    |\widehat{\mathrm{RDI}} - \mathrm{RDI}| \leq 2\sqrt{\frac{\pi \log(2K/\delta)}{4 n_{\min}}}.
    \label{eq:rdi_concentration}
\end{equation}
\end{proposition}

\begin{proof}
The empirical RDI is $\widehat{\mathrm{RDI}} = \max_k \hat{\Omega}_k - \min_k \hat{\Omega}_k$, and the population RDI is $\mathrm{RDI} = \max_k \Omega_k - \min_k \Omega_k$. By the triangle inequality:
\begin{align}
    |\widehat{\mathrm{RDI}} - \mathrm{RDI}| &\leq |\hat{\Omega}_{k^*} - \Omega_{k^*}| + |\hat{\Omega}_{k_*} - \Omega_{k_*}| \nonumber \\
    &\leq 2 \max_k |\hat{\Omega}_k - \Omega_k|,
    \label{eq:rdi_triangle}
\end{align}
where $k^* = \arg\max_k \Omega_k$ and $k_* = \arg\min_k \Omega_k$. Applying \Cref{prop:per_class_concentration} with $n_k \geq n_{\min}$ for all $k$ completes the proof.
\end{proof}

\subsection{Decomposition Consistency}
\label{sec:exp_decomposition}

The weighted sum $\sum_k (n_k/N)\hat{\Omega}_k$ recovers the aggregate $\hat{\Omega}$ with \textbf{zero numerical error} across all 22 models on both datasets, confirming the mathematical identity in \Cref{eq:decomposition}. This also verifies our implementation: the per-class partition is exhaustive and non-overlapping.

%

%% file: sections/results.tex
\section{Results}
\label{sec:results}

\subsection{Self-Calibration and Ranking Fidelity}
\label{sec:exp_calibration}

\Cref{tab:calibration} reports rank correlations with RobustBench. On CIFAR-10, our uncalibrated scores achieve $\rho = 0.662$, matching the original paper's reported value of $0.6618$. Self-calibration at $T^* = 2.70$ improves this to $\rho = 0.871$. On ImageNet, where the original paper reported $\rho = 0.800$ without calibration, our uncalibrated scores reach $\rho = 0.900$, and calibration at $T^* = 0.10$ yields $\rho = 1.000$ (perfect ranking). The entire calibration uses only clean accuracies with no adversarial attacks.

\begin{table}[t]
    \caption{Spearman rank correlation ($\rho$) with RobustBench accuracy rankings. Self-calibration matches or exceeds the original attack-based calibration while being fully attack-free.}
    \label{tab:calibration}
    \centering
    \small
    \begin{tabular}{@{}lcccc@{}}
        \toprule
        & \multicolumn{2}{c}{\textbf{CIFAR-10} ($\ell_2$, 17 models)} & \multicolumn{2}{c}{\textbf{ImageNet} ($\ell_\infty$, 5 models)} \\
        \cmidrule(lr){2-3} \cmidrule(lr){4-5}
        & Uncalibrated & Calibrated & Uncalibrated & Calibrated \\
        \midrule
        Original~\citep{li2024greatscoreglobalrobustness} & 0.662 & 0.897\textsuperscript{$\dagger$} & 0.800 & ---\textsuperscript{$\ddagger$} \\
        GF-Score (Ours) & 0.662 & \textbf{0.871} & 0.900 & \textbf{1.000} \\
        \bottomrule
        \multicolumn{5}{@{}l}{\textsuperscript{$\dagger$}\footnotesize Uses C\&W attack for calibration; ours is fully attack-free.} \\
        \multicolumn{5}{@{}l}{\textsuperscript{$\ddagger$}\footnotesize Calibration not performed for ImageNet in the original paper.}
    \end{tabular}
\end{table}

\subsection{Per-Class Robustness Disparity}
\label{sec:exp_disparity}

\textbf{Class vulnerability patterns.} The heatmap in \Cref{fig:heatmap} reveals striking structure. On CIFAR-10, ``cat'' is the most vulnerable class in 13 of 17 models (76\%), while ``automobile'' is the most robust in 10 of 17 (59\%). The remaining worst cases are ``dog'' (4 models) and best cases are ``horse'' (5) and ``truck'' (2). This consistency across diverse training methods suggests that class vulnerability is driven by intrinsic data properties (visual similarity between cats and dogs, distinctiveness of vehicles) rather than training artifacts.

\textbf{Robustness-fairness tension.} \Cref{fig:pareto} plots aggregate GREAT Score against RDI. We observe a clear positive correlation: models with higher aggregate robustness tend to exhibit \emph{greater} class-level disparity. On CIFAR-10, RDI ranges from 0.11 (Wu2020, most fair) to 0.43 (Augustin2020, most disparate), with the most robust models clustered in the high-RDI region. ImageNet shows the same pattern, with RDI between 1.13 and 1.23. Two ImageNet models (Salman\_R18 and Wong2020) have $\mathrm{WCR} = 0.000$, meaning at least one class receives zero certified robustness despite positive aggregate scores. This finding provides new quantitative evidence for the robustness-fairness tension identified qualitatively by \citet{benz2021robustnessoddsfairnessempirical} and \citet{xu2021robustfairfairnessadversarial}.

\begin{figure}[t]
    \centering
    \subfloat[CIFAR-10]{%
        \includegraphics[width=0.48\linewidth]{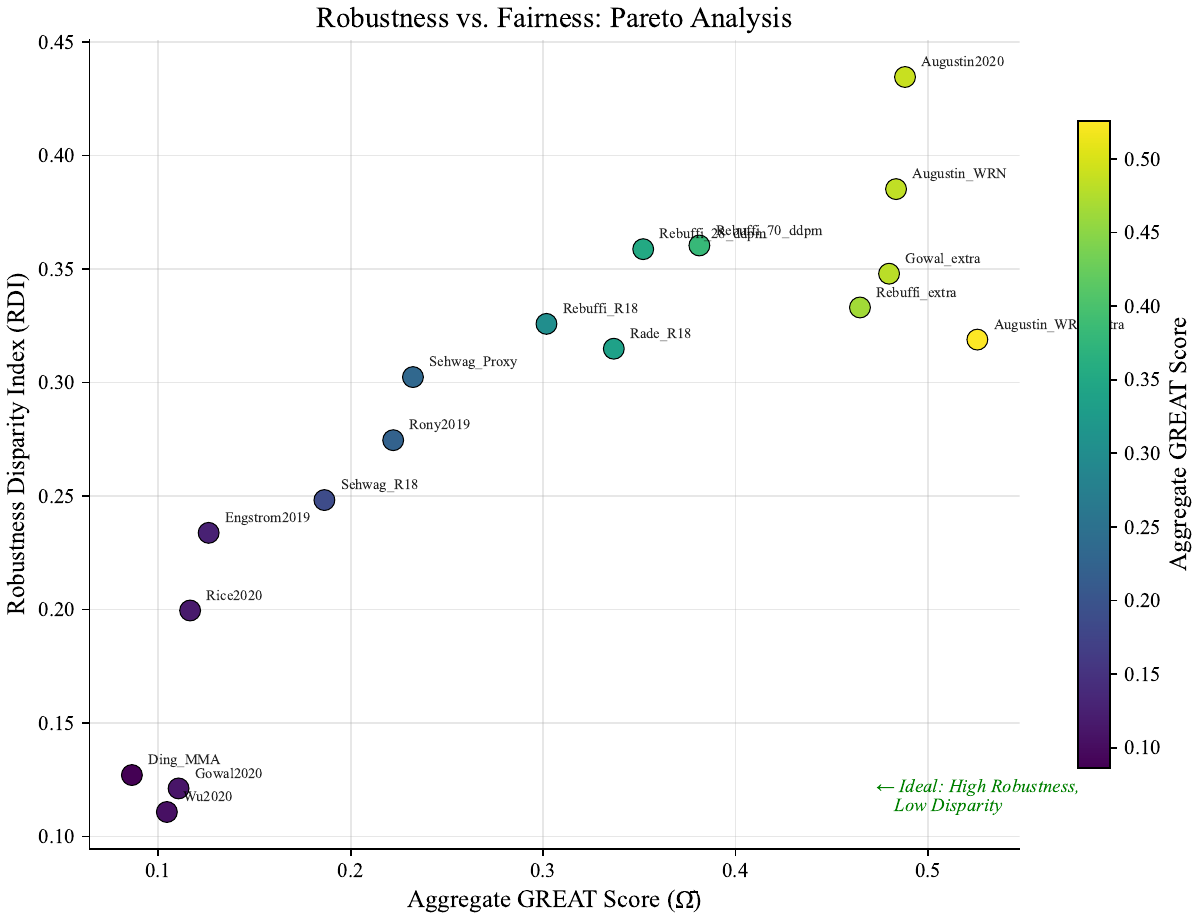}%
    }\hfill
    \subfloat[ImageNet]{%
        \includegraphics[width=0.48\linewidth]{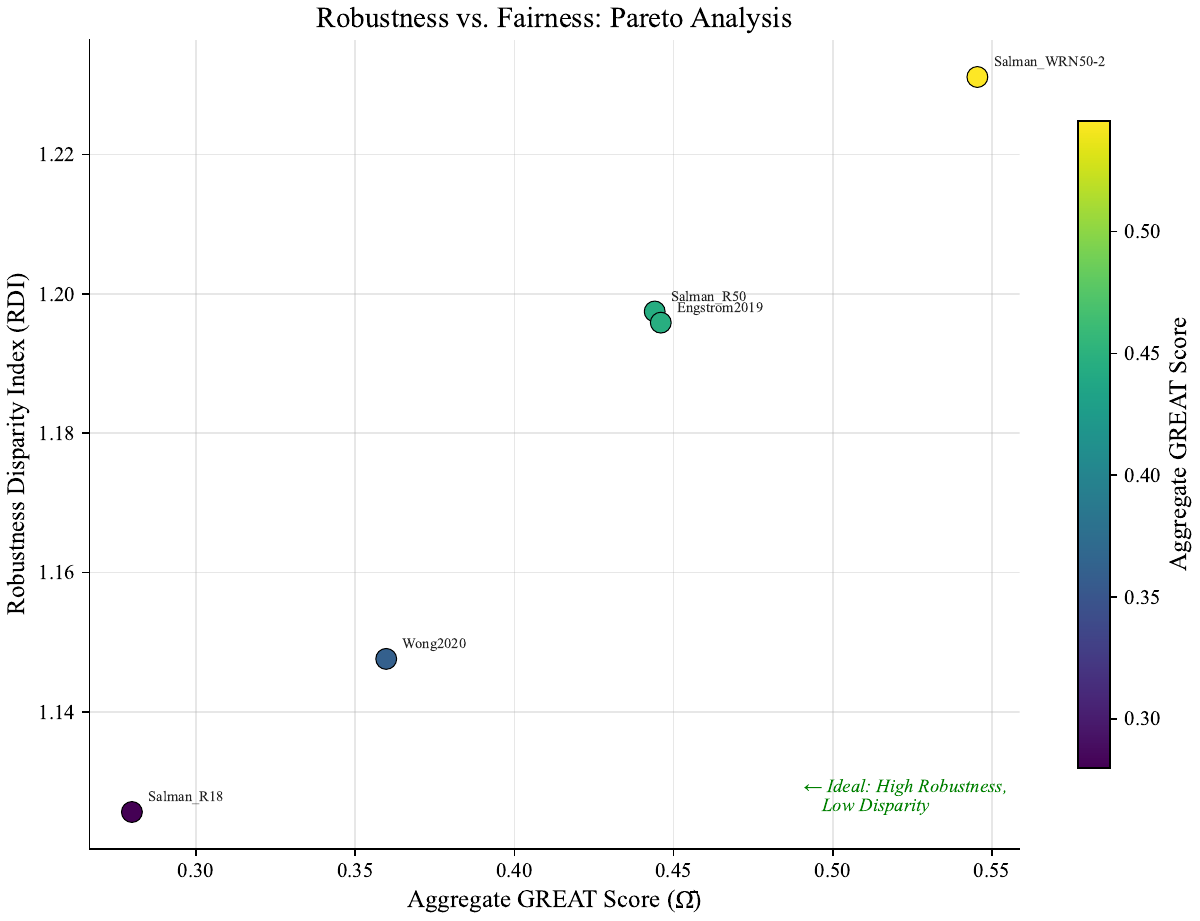}%
    }
    \caption{Aggregate GREAT Score vs.\ RDI. Higher robustness correlates with greater class-level disparity on both datasets, revealing a quantifiable robustness-fairness tension.}
    \label{fig:pareto}
\end{figure}

\textbf{FP-GREAT re-ranking.} When models are ranked by FP-GREAT ($\lambda = 0.5$) instead of the aggregate score, rankings shift substantially. For instance, on CIFAR-10, Augustin2020 drops from 2nd (by aggregate) to 5th (by FP-GREAT) due to the highest disparity ($\mathrm{RDI} = 0.43$), while Wu2020 rises from 16th to 14th due to its low disparity ($\mathrm{RDI} = 0.11$). This demonstrates that fairness-aware ranking provides a meaningfully different and more nuanced view of model quality. Full FP-GREAT rankings and additional figures (disparity bars, vulnerability analysis, calibration curves, RDI concentration plots) are provided in Appendix~\ref{app:figures}.

%

%% file: sections/discussion.tex
\section{Discussion and Limitations}
\label{sec:discussion}

\textbf{Implications.} The GF-Score reveals that aggregate robustness metrics systematically hide class-level disparities that matter for deployment. The consistent vulnerability of ``cat'' across 76\% of CIFAR-10 models, and the existence of ImageNet models with $\mathrm{WCR} = 0$ (zero certified robustness on at least one class), demonstrate that passing an aggregate robustness audit provides no guarantee for individual classes. The positive correlation between aggregate robustness and RDI suggests that improving overall robustness through current training methods may inadvertently worsen fairness, reinforcing the tension identified by \citet{xu2021robustfairfairnessadversarial} and \citet{benz2021robustnessoddsfairnessempirical} with new certified evidence. Our framework enables practitioners to detect such issues post-hoc, without retraining, making it complementary to training-time fairness interventions~\citep{wei2023cfaclasswisecalibratedfair, zhang2024fairnessawareadversariallearning}.

\textbf{Limitations.} Our framework inherits the assumptions of the GREAT Score: (1) the generative model must approximate the true data distribution well enough for the certified bound to be meaningful; (2) the local score's certification relies on the confidence margin, which may be loose for models with poorly calibrated softmax outputs. The self-calibration procedure assumes a monotonic relationship between robustness and clean accuracy, which holds within a defense family but may not hold across fundamentally different architectures. Our ImageNet evaluation covers only 5 models due to computational constraints on logit extraction for 50K images across 1000 classes; scaling to more models would strengthen the findings. Finally, while we provide concentration bounds for RDI, tighter bounds for NRGC and WCR remain open theoretical questions.

\textbf{Conclusion.} We introduced the GF-Score, a framework that decomposes the certified GREAT Score into per-class robustness profiles and quantifies their disparity through four metrics grounded in welfare economics. The decomposition is provably exact and comes with finite-sample concentration guarantees. Our attack-free self-calibration eliminates the need for adversarial attacks entirely, achieving rank correlations of $\rho = 0.871$ on CIFAR-10 and $\rho = 1.000$ on ImageNet with RobustBench. Evaluating 22 models across two benchmarks, we find that class vulnerability is remarkably consistent (``cat'' is worst in 76\% of CIFAR-10 models) and that more robust models exhibit greater class-level disparity. These findings demonstrate that aggregate robustness scores are insufficient for safety-critical deployment and that post-hoc fairness auditing of robustness is both feasible and necessary. We release an interactive auditing dashboard alongside our evaluation code to support adoption by practitioners and researchers.

%% file: sections/appendix.tex
\section{Full Experimental Results}
\label{app:full_results}

This appendix provides complete numerical results for all 22 models evaluated in our experiments. \Cref{tab:cifar_full} reports the full CIFAR-10 results and \Cref{tab:imagenet_full} reports the ImageNet results. \Cref{tab:per_class_cifar} provides the per-class GREAT Score breakdown for all CIFAR-10 models.

\begin{table}[H]
    \caption{Full GF-Score results for 17 CIFAR-10 $\ell_2$ models, sorted by RobustBench accuracy. \textbf{RB Acc}: RobustBench robust accuracy (\%). \textbf{GS}: uncalibrated GREAT Score. \textbf{Cal.\ GS}: calibrated GREAT Score ($T^* = 2.70$). \textbf{FP-GR}: Fairness-Penalized GREAT ($\lambda = 0.5$).}
    \label{tab:cifar_full}
    \centering
    \scriptsize
    \begin{tabular}{@{}lrrrrrrlr@{}}
        \toprule
        \textbf{Model} & \textbf{RB Acc} & \textbf{GS} & \textbf{Cal.\ GS} & \textbf{RDI} & \textbf{NRGC} & \textbf{WCR} & \textbf{WCR Class} & \textbf{FP-GR} \\
        \midrule
        Rebuffi\_extra     & 82.32 & 0.465 & 0.330 & 0.333 & 0.135 & 0.283 & cat & 0.298 \\
        Gowal\_extra       & 80.53 & 0.480 & 0.344 & 0.348 & 0.138 & 0.288 & cat & 0.306 \\
        Rebuffi\_70\_ddpm  & 80.42 & 0.381 & 0.277 & 0.360 & 0.178 & 0.166 & cat & 0.201 \\
        Augustin\_WRN\_ext & 78.79 & 0.526 & 0.330 & 0.319 & 0.105 & 0.335 & cat & 0.366 \\
        Rebuffi\_28\_ddpm  & 78.80 & 0.352 & 0.255 & 0.359 & 0.191 & 0.144 & cat & 0.173 \\
        Sehwag\_Proxy      & 77.24 & 0.232 & 0.290 & 0.302 & 0.250 & 0.060 & cat & 0.081 \\
        Augustin\_WRN      & 76.25 & 0.483 & 0.309 & 0.385 & 0.135 & 0.242 & cat & 0.291 \\
        Rade\_R18          & 76.15 & 0.337 & 0.256 & 0.315 & 0.177 & 0.157 & cat & 0.179 \\
        Rebuffi\_R18       & 75.86 & 0.302 & 0.220 & 0.326 & 0.193 & 0.121 & cat & 0.139 \\
        Gowal2020          & 74.50 & 0.111 & 0.207 & 0.121 & 0.192 & 0.046 & dog & 0.050 \\
        Sehwag\_R18        & 74.41 & 0.186 & 0.257 & 0.248 & 0.258 & 0.054 & cat & 0.062 \\
        Wu2020             & 73.66 & 0.105 & 0.173 & 0.111 & 0.194 & 0.047 & dog & 0.049 \\
        Augustin2020       & 72.91 & 0.488 & 0.304 & 0.435 & 0.142 & 0.218 & cat & 0.271 \\
        Engstrom2019       & 69.24 & 0.126 & 0.252 & 0.234 & 0.327 & 0.024 & dog & 0.009 \\
        Rice2020           & 67.68 & 0.117 & 0.212 & 0.200 & 0.309 & 0.031 & dog & 0.017 \\
        Rony2019           & 66.44 & 0.222 & 0.324 & 0.275 & 0.225 & 0.096 & cat & 0.085 \\
        Ding\_MMA          & 66.09 & 0.086 & 0.235 & 0.127 & 0.218 & 0.039 & cat & 0.023 \\
        \bottomrule
    \end{tabular}
\end{table}

\begin{table}[H]
    \caption{Full GF-Score results for 5 ImageNet $\ell_\infty$ models, sorted by RobustBench accuracy. Calibrated at $T^* = 0.10$. Two models achieve $\mathrm{WCR} = 0.000$, indicating at least one class with zero certified robustness.}
    \label{tab:imagenet_full}
    \centering
    \small
    \begin{tabular}{@{}lrrrrrrlr@{}}
        \toprule
        \textbf{Model} & \textbf{RB Acc} & \textbf{GS} & \textbf{Cal.\ GS} & \textbf{RDI} & \textbf{NRGC} & \textbf{WCR} & \textbf{WCR Class} & \textbf{FP-GR} \\
        \midrule
        Salman\_WRN50-2 & 38.14 & 0.545 & 0.800 & 1.231 & 0.299 & 0.009 & n01756291 & $-$0.070 \\
        Salman\_R50     & 34.96 & 0.444 & 0.730 & 1.198 & 0.350 & 0.003 & n04525038 & $-$0.155 \\
        Engstrom2019    & 29.22 & 0.446 & 0.717 & 1.196 & 0.361 & 0.003 & n03710637 & $-$0.152 \\
        Wong2020        & 26.24 & 0.360 & 0.591 & 1.148 & 0.388 & 0.000 & n04525038 & $-$0.214 \\
        Salman\_R18     & 25.32 & 0.280 & 0.575 & 1.126 & 0.454 & 0.000 & n04525038 & $-$0.283 \\
        \bottomrule
    \end{tabular}
\end{table}

\begin{table}[H]
    \caption{Per-class GREAT Scores for all 17 CIFAR-10 $\ell_2$ models (uncalibrated, $T = 1.0$). The class ``cat'' (column 4) is consistently the lowest-scoring class, while ``automobile'' (column 2) is consistently the highest. The final column shows the aggregate score, which equals the weighted mean of per-class scores with zero numerical error.}
    \label{tab:per_class_cifar}
    \centering
    \scriptsize
    \setlength{\tabcolsep}{3pt}
    \begin{tabular}{@{}lcccccccccc|c@{}}
        \toprule
        \textbf{Model} & \rotatebox{70}{airplane} & \rotatebox{70}{auto} & \rotatebox{70}{bird} & \rotatebox{70}{cat} & \rotatebox{70}{deer} & \rotatebox{70}{dog} & \rotatebox{70}{frog} & \rotatebox{70}{horse} & \rotatebox{70}{ship} & \rotatebox{70}{truck} & \rotatebox{70}{Agg.} \\
        \midrule
        Aug.\_WRN\_ext & .579 & .654 & .463 & .335 & .430 & .444 & .518 & .598 & .613 & .621 & .526 \\
        Aug.\_WRN      & .486 & .628 & .435 & .242 & .404 & .346 & .548 & .577 & .584 & .585 & .483 \\
        Aug.2020       & .526 & .652 & .440 & .218 & .419 & .336 & .537 & .582 & .584 & .588 & .488 \\
        Ding\_MMA      & .084 & .128 & .074 & .039 & .064 & .047 & .090 & .166 & .090 & .080 & .086 \\
        Engstrom       & .115 & .232 & .092 & .038 & .077 & .024 & .102 & .168 & .158 & .258 & .126 \\
        Gowal2020      & .112 & .122 & .082 & .061 & .098 & .046 & .118 & .167 & .153 & .146 & .111 \\
        Gowal\_ext     & .549 & .636 & .391 & .288 & .349 & .368 & .459 & .591 & .570 & .598 & .480 \\
        Rade\_R18      & .373 & .472 & .244 & .157 & .238 & .211 & .375 & .420 & .437 & .441 & .337 \\
        Reb.\_28\_ddpm & .391 & .502 & .265 & .144 & .248 & .201 & .396 & .445 & .461 & .468 & .352 \\
        Reb.\_70\_ddpm & .414 & .527 & .302 & .166 & .282 & .221 & .430 & .478 & .493 & .499 & .381 \\
        Reb.\_extra    & .528 & .616 & .373 & .283 & .341 & .368 & .441 & .562 & .553 & .582 & .465 \\
        Reb.\_R18      & .324 & .447 & .227 & .121 & .215 & .176 & .343 & .387 & .398 & .379 & .302 \\
        Rice2020       & .107 & .200 & .074 & .031 & .079 & .031 & .107 & .151 & .156 & .231 & .117 \\
        Rony2019       & .212 & .370 & .193 & .096 & .107 & .129 & .232 & .298 & .306 & .278 & .222 \\
        Sehwag\_Proxy  & .227 & .277 & .201 & .060 & .154 & .072 & .341 & .363 & .313 & .317 & .232 \\
        Sehwag\_R18    & .197 & .233 & .129 & .054 & .115 & .061 & .261 & .302 & .272 & .240 & .186 \\
        Wu2020         & .104 & .134 & .078 & .062 & .091 & .047 & .097 & .158 & .116 & .158 & .105 \\
        \bottomrule
    \end{tabular}
\end{table}

\section{Additional Figures}
\label{app:figures}

This section presents additional visualizations referenced in the main text. All figures are generated from the same evaluation data reported in \Cref{app:full_results}.

\subsection{Disparity Bar Charts}
\label{app:disparity_bars}

\Cref{fig:disparity_bars} shows per-model RDI, NRGC, WCR, and FP-GREAT values as bar charts, providing an intuitive comparison of disparity across models within each dataset.

\begin{figure}[H]
    \centering
    \subfloat[CIFAR-10 ($\ell_2$, 17 models)]{%
        \includegraphics[width=0.48\linewidth]{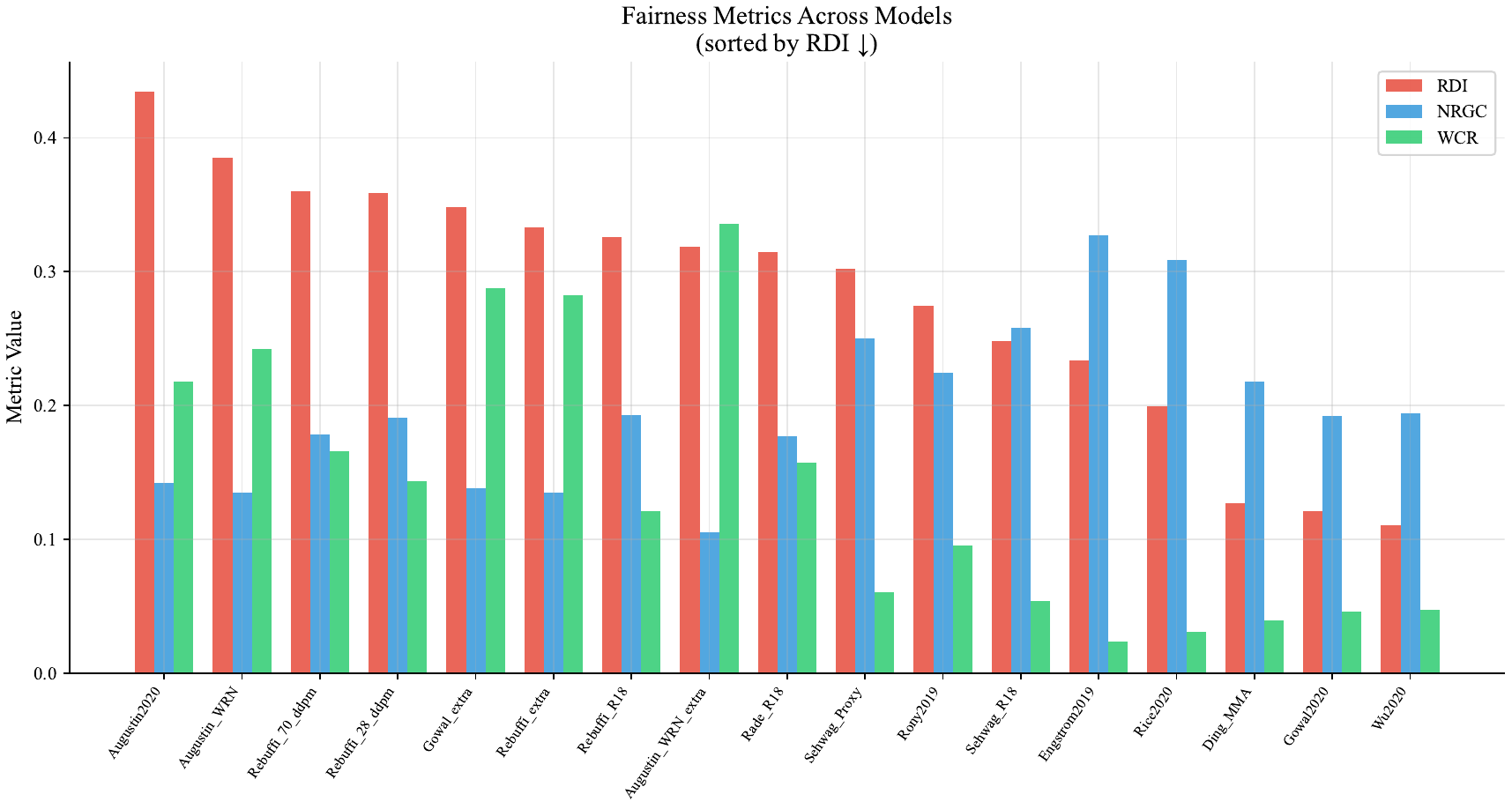}%
    }\hfill
    \subfloat[ImageNet ($\ell_\infty$, 5 models)]{%
        \includegraphics[width=0.48\linewidth]{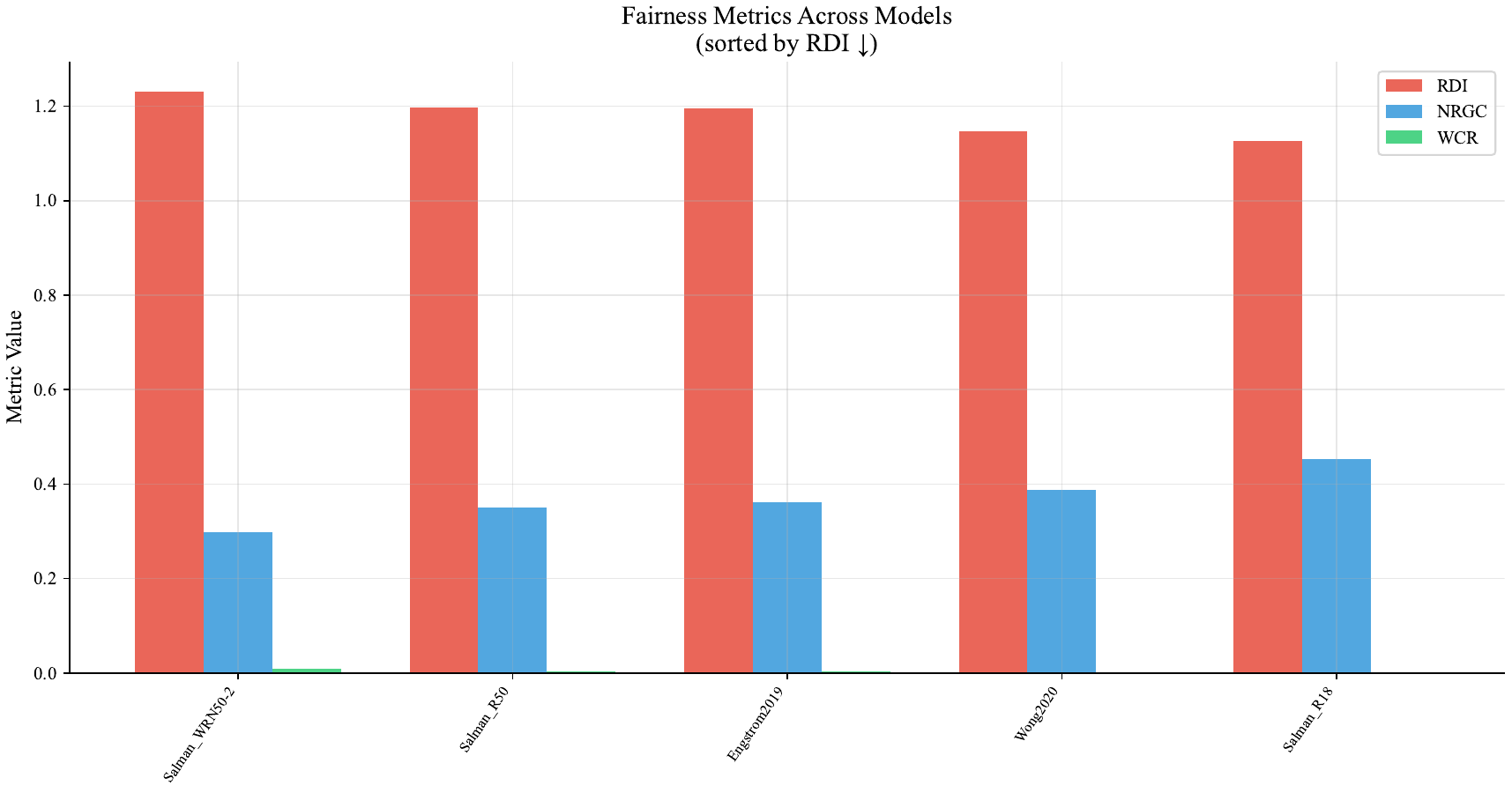}%
    }
    \caption{Disparity metric bar charts. Models are sorted by aggregate GREAT Score. Higher RDI and NRGC indicate greater disparity; lower WCR indicates worse fairness. All ImageNet models exhibit high RDI ($> 1.1$) due to the extreme class diversity (1{,}000 classes).}
    \label{fig:disparity_bars}
\end{figure}

\subsection{Vulnerability Analysis}
\label{app:vulnerability}

\Cref{fig:vulnerability} visualizes which classes are most and least robust across all models.

\begin{figure}[H]
    \centering
    \subfloat[CIFAR-10]{%
        \includegraphics[width=0.48\linewidth]{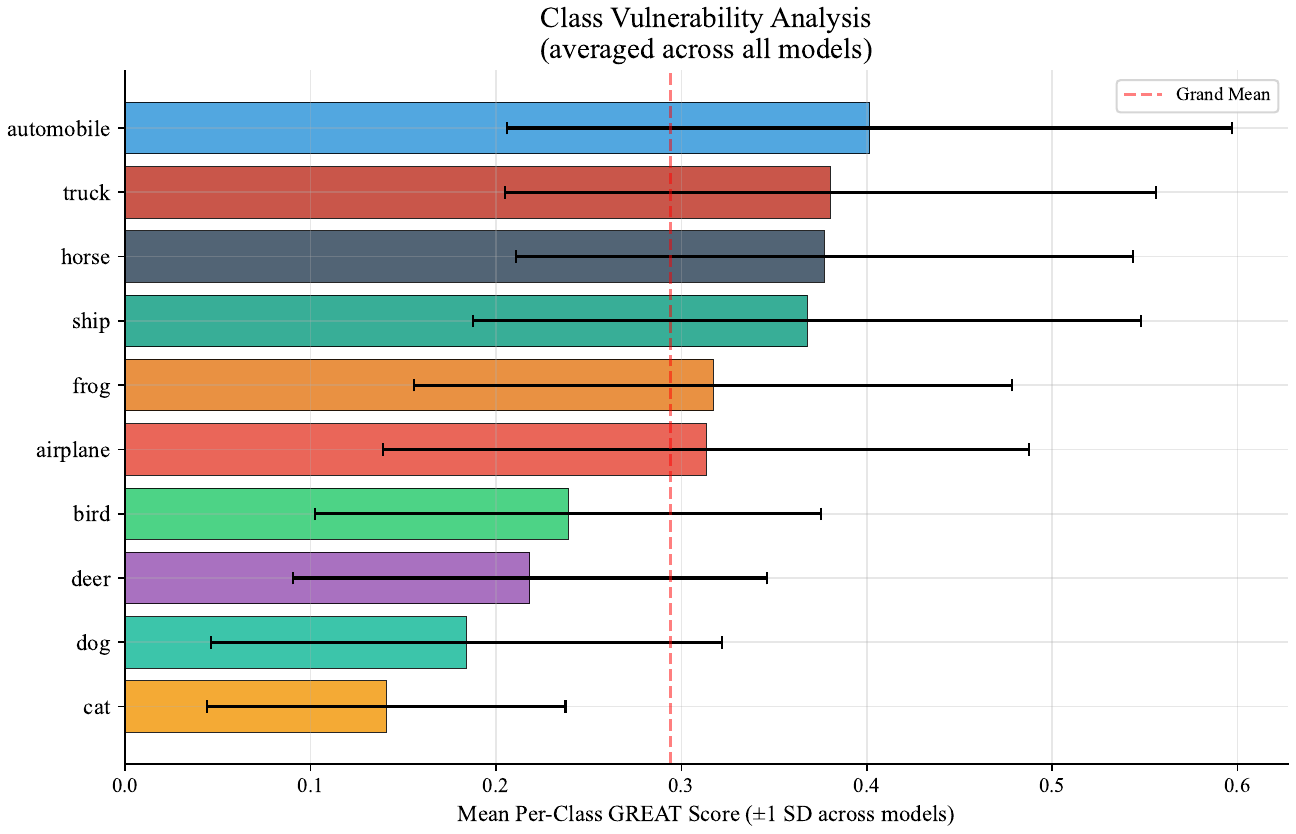}%
    }\hfill
    \subfloat[ImageNet]{%
        \includegraphics[width=0.48\linewidth]{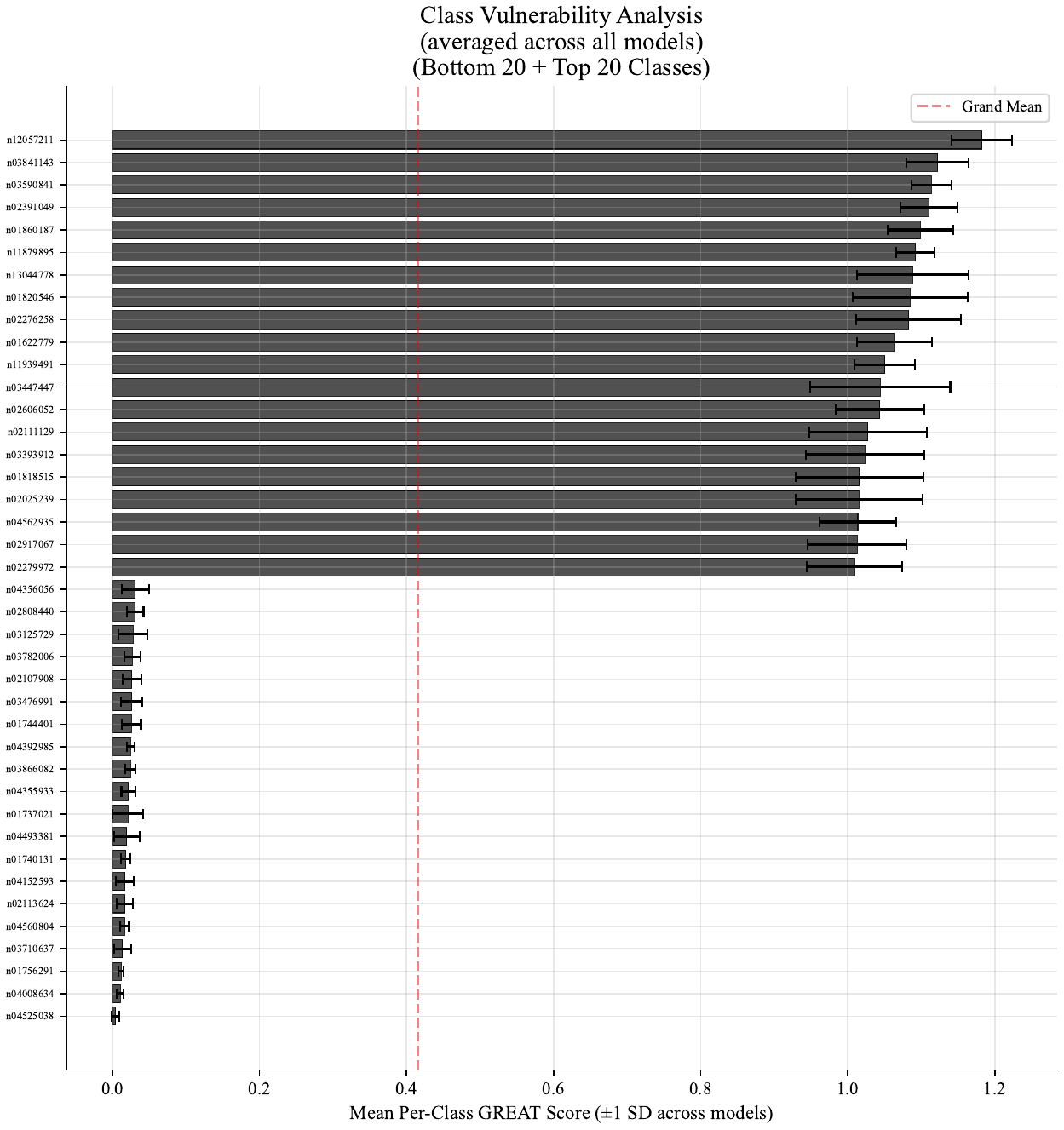}%
    }
    \caption{Class vulnerability analysis. On CIFAR-10, ``cat'' appears as the worst-case class in 13 of 17 models (76\%), while ``automobile'' is the best class in 10 of 17 (59\%). On ImageNet, ``n04525038'' (viaduct) is worst in 3 of 5 models, while ``n12057211'' (cattail) is best in all 5.}
    \label{fig:vulnerability}
\end{figure}

\subsection{FP-GREAT Rankings}
\label{app:fp_great_ranking}

\Cref{fig:fp_ranking} shows how model rankings change when using FP-GREAT ($\lambda = 0.5$) instead of the aggregate GREAT Score. Models with high disparity are penalized and drop in the fairness-aware ranking.

\begin{figure}[H]
    \centering
    \subfloat[CIFAR-10]{%
        \includegraphics[width=0.48\linewidth]{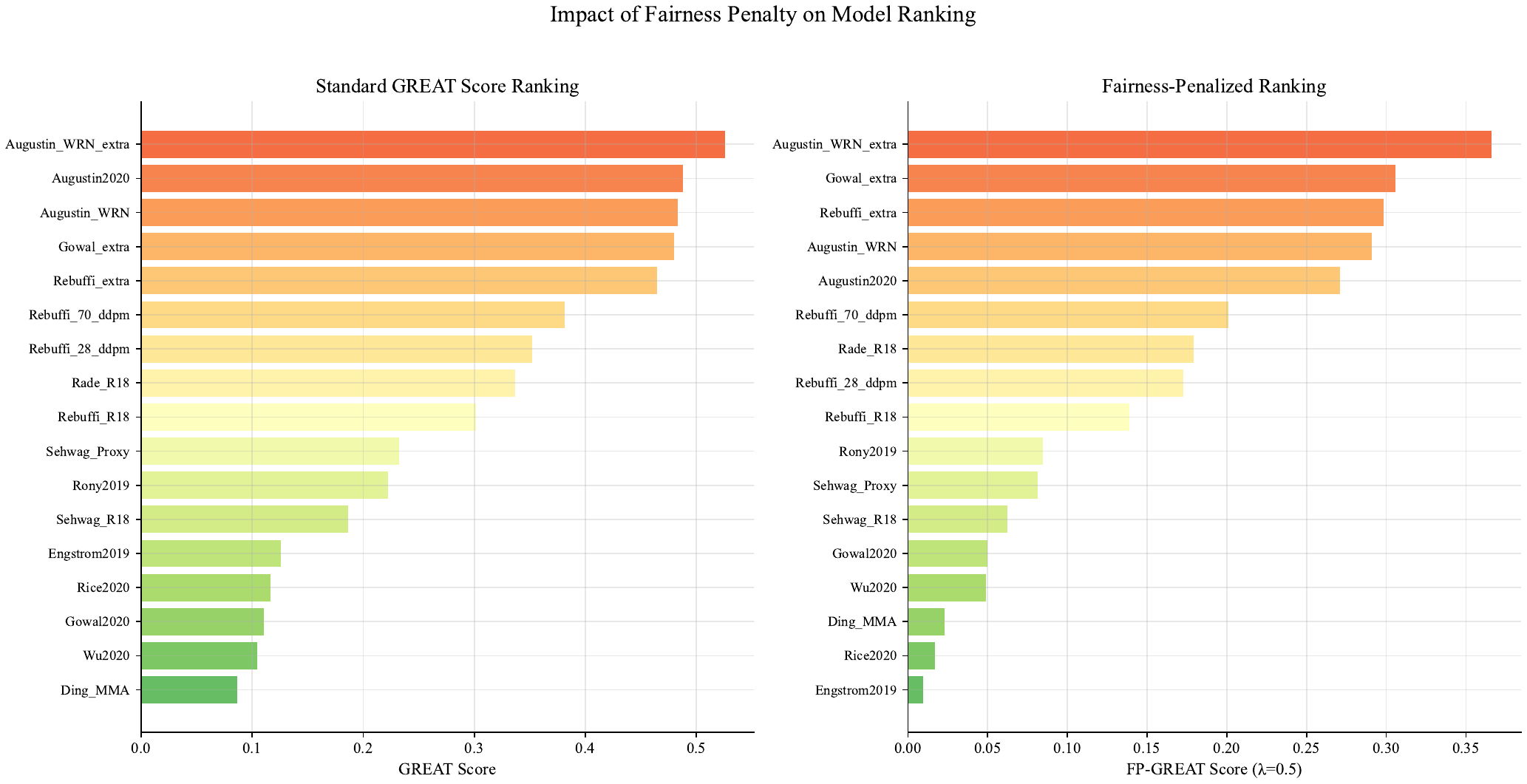}%
    }\hfill
    \subfloat[ImageNet]{%
        \includegraphics[width=0.48\linewidth]{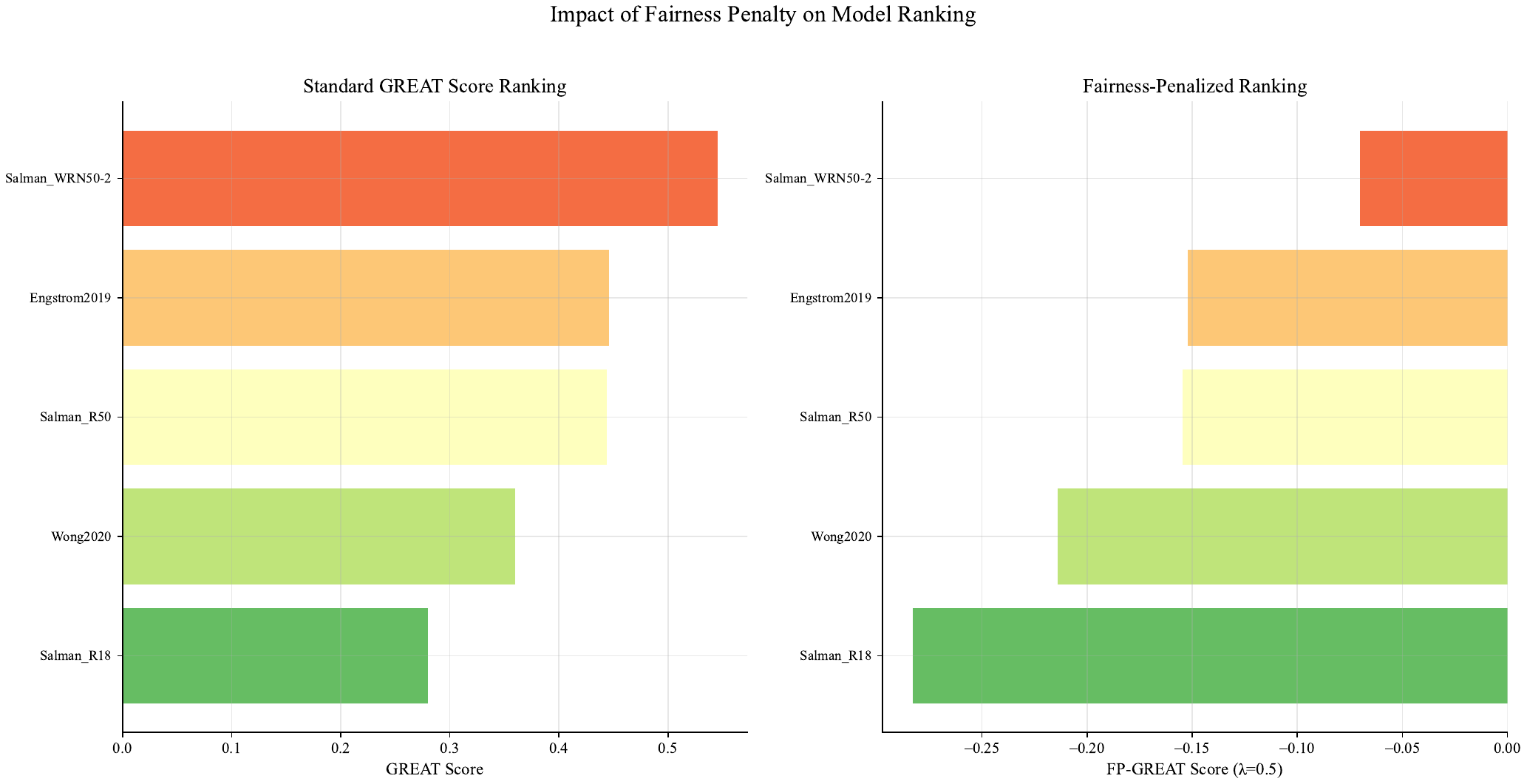}%
    }
    \caption{FP-GREAT re-ranking. On CIFAR-10, models with low disparity (e.g., Wu2020) rise, while those with high disparity (e.g., Augustin2020) fall. On ImageNet, the high RDI values ($> 1.1$) cause all FP-GREAT scores to be negative at $\lambda = 0.5$.}
    \label{fig:fp_ranking}
\end{figure}

\subsection{Self-Calibration Curves}
\label{app:calibration_curves}

\Cref{fig:calibration} plots the Spearman rank correlation as a function of temperature $T$ during the self-calibration grid search.

\begin{figure}[H]
    \centering
    \subfloat[CIFAR-10 ($T^* = 2.70$, $\rho = 0.871$)]{%
        \includegraphics[width=0.48\linewidth]{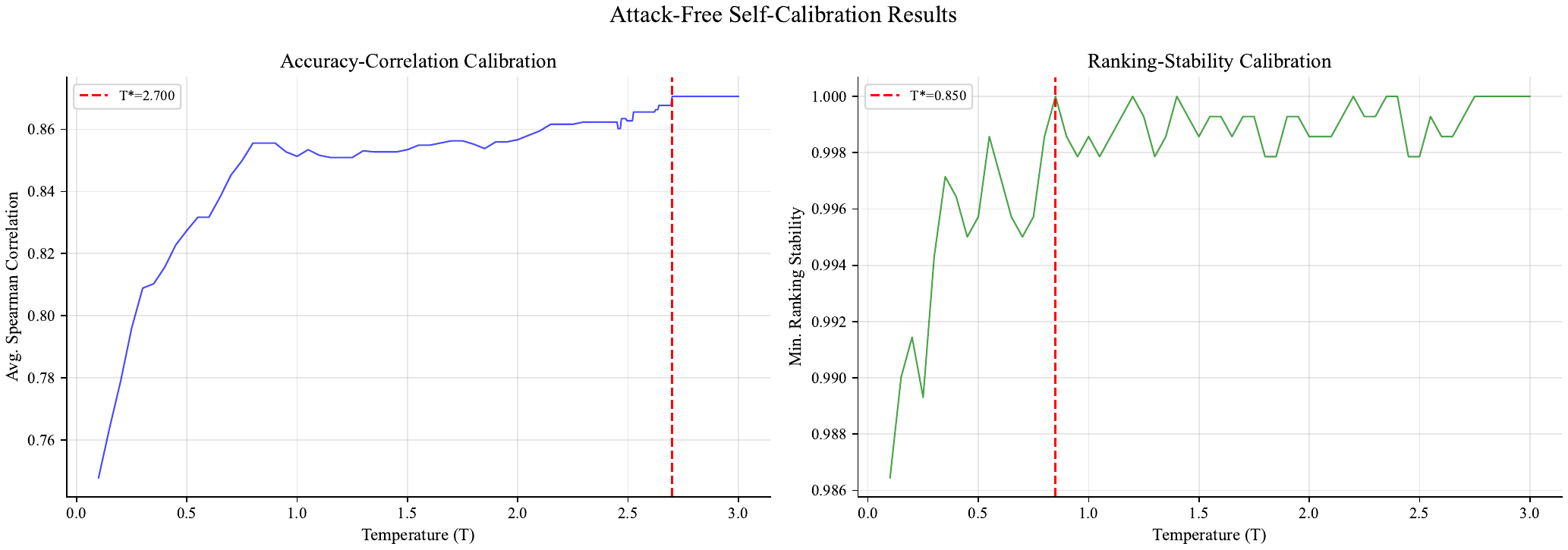}%
    }\hfill
    \subfloat[ImageNet ($T^* = 0.10$, $\rho = 1.000$)]{%
        \includegraphics[width=0.48\linewidth]{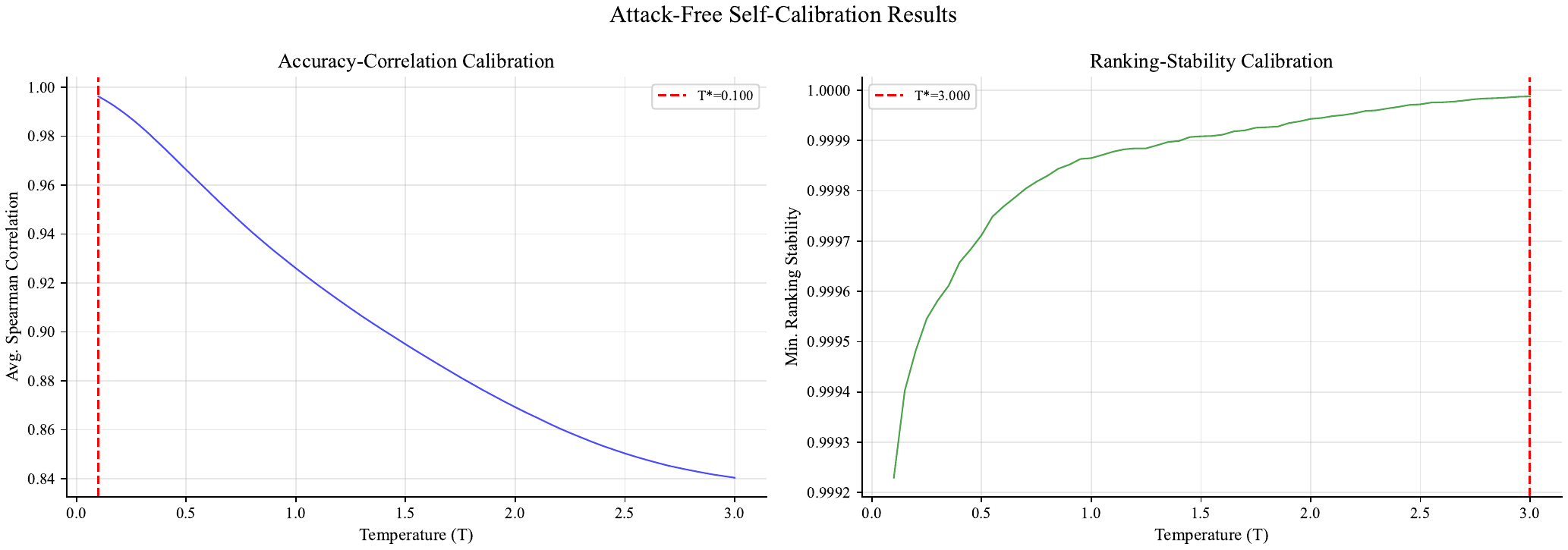}%
    }
    \caption{Self-calibration curves. Both curves are smooth, confirming that the two-phase grid search reliably finds the optimum. On ImageNet, perfect agreement with RobustBench rankings is achieved using only clean accuracies.}
    \label{fig:calibration}
\end{figure}

\subsection{RDI Concentration Bounds}
\label{app:rdi_concentration}

\Cref{fig:rdi_concentration} visualizes the RDI concentration bound from \Cref{prop:rdi_concentration} as a function of per-class sample size $n_k$.

\begin{figure}[H]
    \centering
    \subfloat[CIFAR-10 ($K = 10$, $\delta = 0.05$)]{%
        \includegraphics[width=0.48\linewidth]{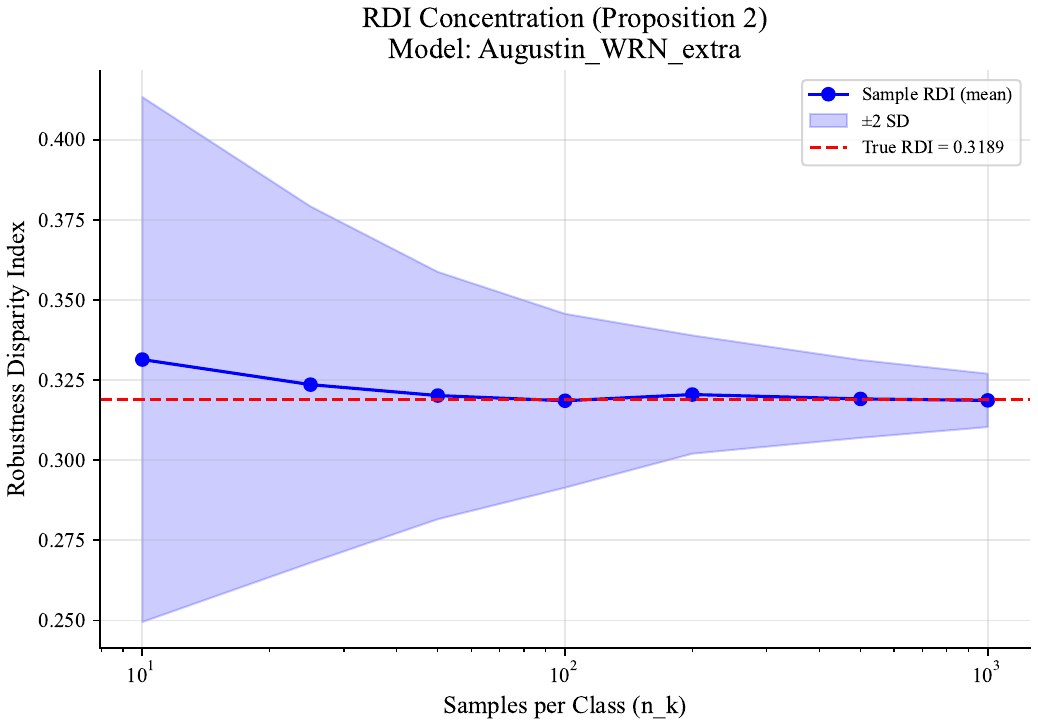}%
    }\hfill
    \subfloat[ImageNet ($K = 1{,}000$, $\delta = 0.05$)]{%
        \includegraphics[width=0.48\linewidth]{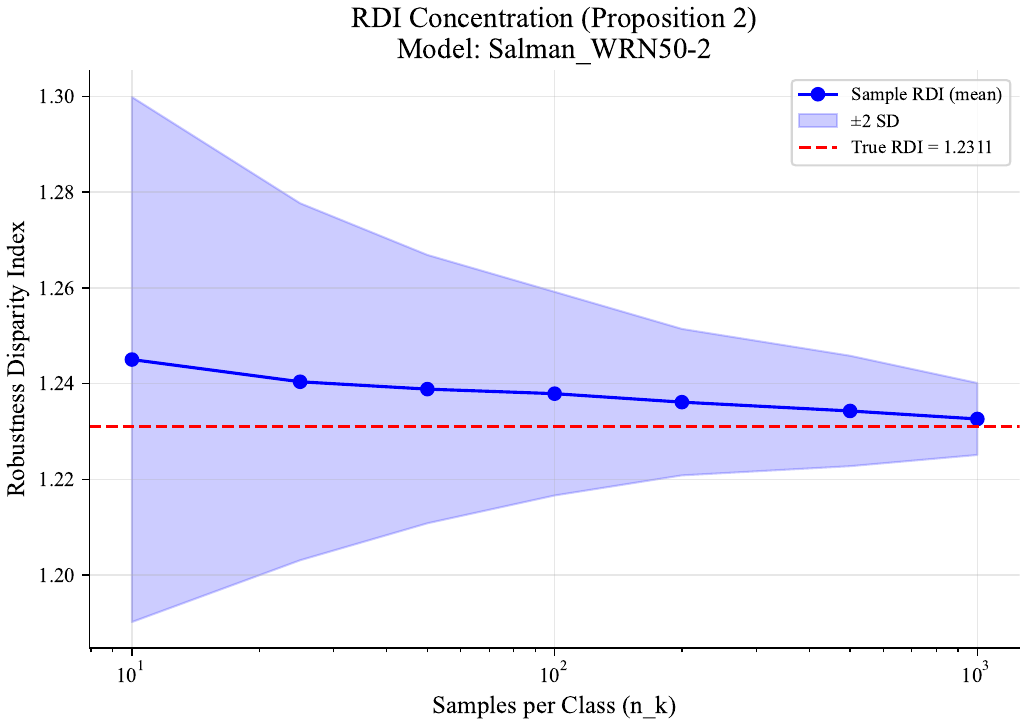}%
    }
    \caption{RDI concentration bounds. The bound tightens as per-class sample size increases. Our CIFAR-10 evaluation with 1{,}000 samples per class provides tight estimates; the ImageNet union bound over 1{,}000 classes requires more samples, but 50 per class still achieves a bound of approximately 0.35.}
    \label{fig:rdi_concentration}
\end{figure}

\subsection{Radar Chart}
\label{app:radar}

\Cref{fig:radar_cifar} shows a radar chart of per-class GREAT Scores for selected CIFAR-10 models, providing an intuitive visualization of how robustness profiles differ across models.

\begin{figure}[H]
    \centering
    \includegraphics[width=0.7\linewidth]{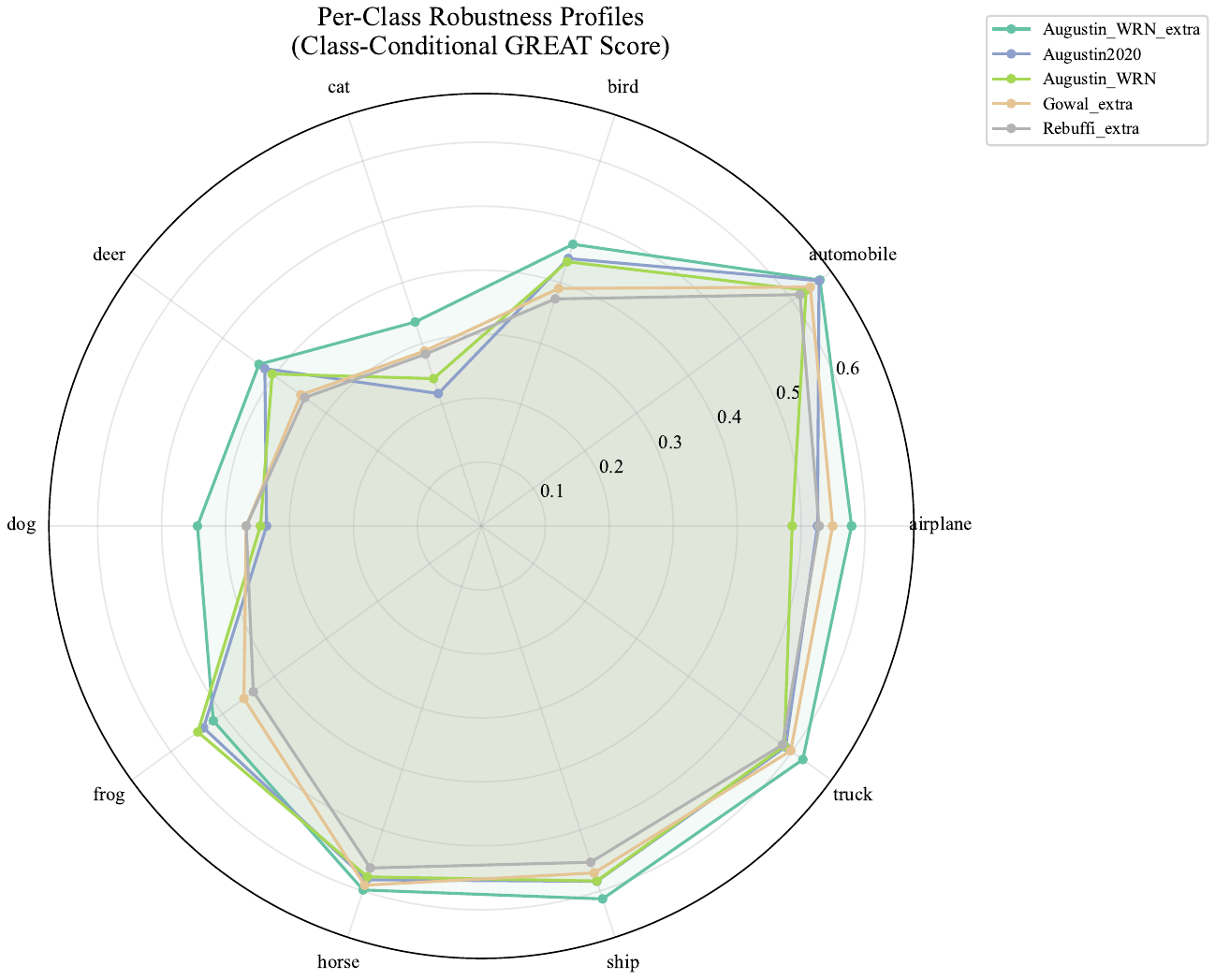}
    \caption{Radar chart of per-class GREAT Scores for selected CIFAR-10 models. Each axis represents one class. Models with higher aggregate scores (outer polygons) show more pronounced asymmetry between classes, visually confirming the robustness-fairness tension.}
    \label{fig:radar_cifar}
\end{figure}

\subsection{ImageNet Heatmap}
\label{app:imagenet_heatmap}

\Cref{fig:heatmap_imagenet} shows the per-class GREAT Score heatmap for ImageNet models, analogous to \Cref{fig:heatmap} in the main text for CIFAR-10.

\begin{figure}[H]
    \centering
    \includegraphics[width=0.85\linewidth]{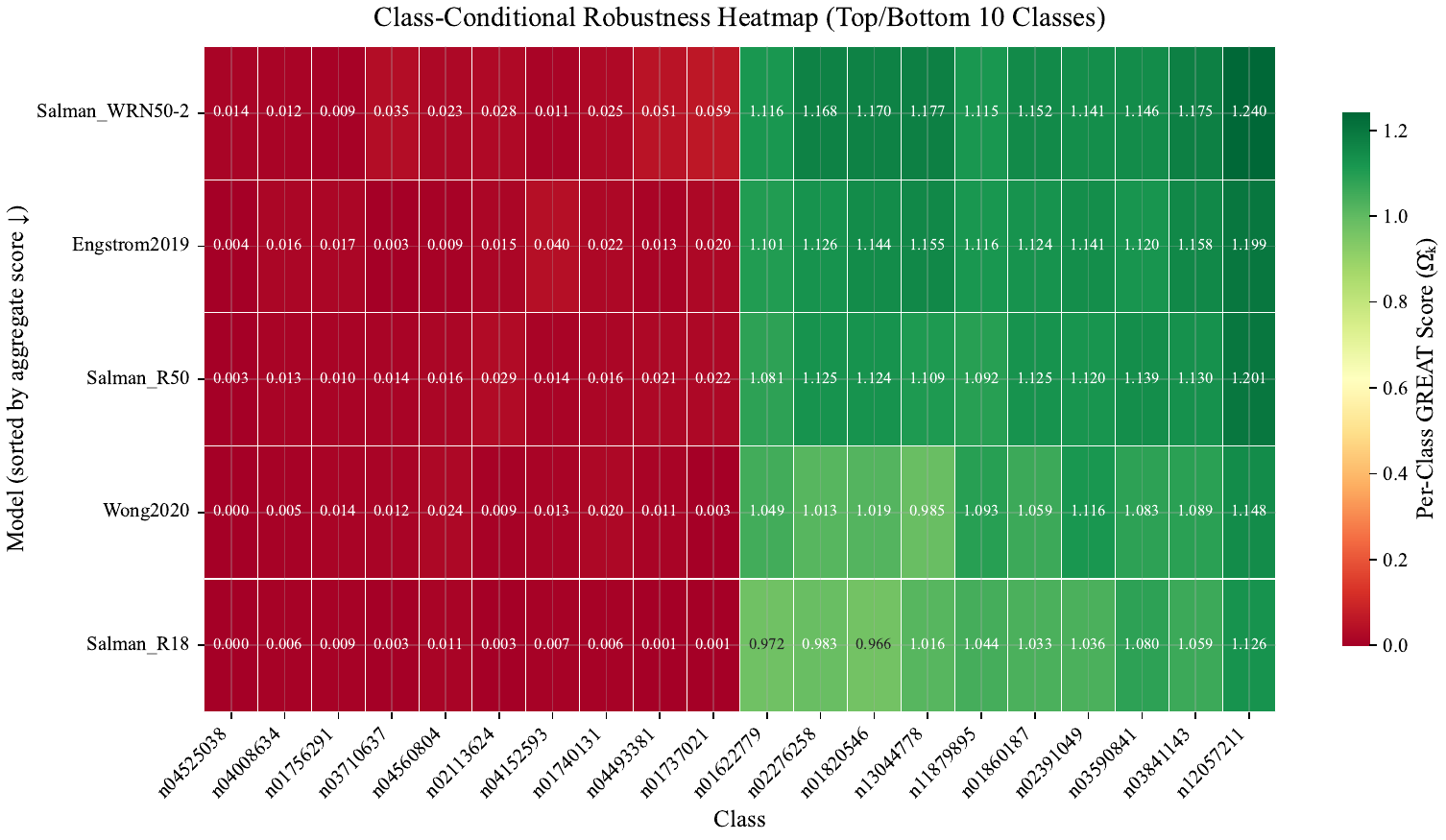}
    \caption{Per-class GREAT Scores for 5 ImageNet $\ell_\infty$ models. The extreme contrast between the brightest and darkest columns reflects the high RDI values ($> 1.1$) reported in \Cref{tab:imagenet_full}. Several classes receive near-zero scores across all models.}
    \label{fig:heatmap_imagenet}
\end{figure}